\documentclass[letterpaper, 10 pt, conference]{ieeeconf}  

\usepackage{amsmath}
\usepackage{cite}
\usepackage{amssymb}
\usepackage{subfigure}
\usepackage{float}
\usepackage{pifont}
\usepackage[labelfont=it,textfont=it]{caption}
\usepackage{color}

\newtheorem{thm}{Theorem}
\newtheorem{lem}[thm]{Lemma}

\newtheorem{defn}{Definition}

\newtheorem{rem}{Remark}

\usepackage{graphicx}
\usepackage{subcaption}

\IEEEoverridecommandlockouts                              

\title{\LARGE \bf
Lateral String Stability for Vehicle Platoons
}

\author{Sixu Li$^{1}$, Swaroop Darbha$^{1,*}$, and Yang Zhou$^{1,*}$
\thanks{$^{1}$Sixu Li, Swaroop Darbha, and Yang Zhou are with
        Texas A\&M University, College Station, TX, 77483, USA
       {\tt\small \{sixuli, dswaroop, yangzhou295\}@tamu.edu}}%
\thanks{$^{*}$Corresponding authors: Swaroop Darbha, Yang Zhou}%
}

\begin{document}

\maketitle
\thispagestyle{empty}
\pagestyle{empty}

\begin{abstract}
Connected and automated vehicle (CAV) platooning promises gains in energy efficiency and traffic throughput and, most critically, in safety. These safety benefits hinge on string stability, which determines how disturbances propagate along a platoon. While longitudinal string stability is well studied, lateral string stability, which governs the propagation of path-tracking errors that can lead to unsafe deviations from the intended path, remains underexplored. Its importance is increasing as autonomous vehicles rely more heavily on onboard sensing and map‑free navigation, where sensor occlusion and dense formations amplify safety risks.
This paper presents a new framework for lateral string stability that directly addresses safety‑critical path‑relative tracking errors and enables consistent comparison across vehicles following the same road geometry. Central to this framework is an arc‑length (Eulerian) viewpoint, a departure from traditional analyses, that clarifies how tracking errors at a given point on the path propagate from one vehicle to the next. A formal definition of lateral string stability is introduced along with two control strategies: an onboard‑sensing‑only controller and a novel learn‑from‑predecessor approach utilizing vehicle-to-vehicle (V2V) communication. We show that onboard sensing alone cannot guarantee attenuation of path-tracking errors, imposing a fundamental safety limitation, whereas V2V communication enables true error attenuation.



\end{abstract}

\section{INTRODUCTION}

Disturbance propagation in traffic streams is a key factor for stability, safety, and efficiency. This phenomenon is captured by \textit{string stability}, which describes whether disturbances attenuate along a platoon. The concept originated in early decentralized control work by Chu \cite{chu1974decentralized}, and was later formalized by Darbha and Hedrick \cite{swaroop1994string,swaroop2002string} to analyze spacing error propagation in automated vehicle (AV) platoons under the California PATH program. Since then, longitudinal string stability has been extensively studied across control architectures and information flow topologies \cite{swaroop1999constant,seiler2004disturbance}. Applications were broadened in reference \cite{tian2025physically,li2024sequencing,li2025nonlinear}, and theoretical generalizations were developed in reference \cite{besselink2017string}.

In contrast, \textit{lateral string stability}, which governs the propagation of steering-induced tracking errors, has received limited attention, even though it is critical for safe AV platooning. With AVs shifting from infrastructure-based guidance (e.g., magnetized lanes; \cite{fenton1976steering}) to onboard sensing and \textit{map-free navigation} \cite{ort2018autonomous}, maintaining precise lateral positioning becomes increasingly difficult. In dense platoons, \textit{sensor occlusion} \cite{zhang2021safe} can further hinder perception, reducing a follower’s ability to detect obstacles, lane boundaries, and dynamic hazards. One remedy is to let the lead vehicle plan a local path for the entire platoon \cite{hassanain2020string,liu2020robust}, with followers either tracking their predecessor’s path or receiving shared data via vehicle-to-vehicle (V2V) communication. Such schemes, however, couple lateral dynamics across vehicles and thus motivate a formal study of lateral string stability.

Several recent studies have begun to address this gap. Alleleijn et al. \cite{alleleijn2014lateral} introduced the idea of lateral string stability but without formal definitions or analysis. Hassanain et al. \cite{hassanain2020string} characterized it through $\mathcal{H}_\infty$ norms of transfer functions, while Liu et al. \cite{liu2020robust} employed $\mathcal{L}_\infty$ attenuation of tracking error vectors and showed that predecessor-only information is insufficient. Somisetty and Darbha \cite{somisetty2024lateral} instead defined lateral string stability as bounded tracking errors and highlighted the benefits of leader information sharing. Nonetheless, these works share important limitations:
\begin{itemize}
    \item Prior analyses focus on specific signals (e.g., yaw rate or predecessor-relative errors) that do not directly constrain path-relative tracking errors, which are safety-critical.
    \item Prior work often assumes straight-line paths and does not systematically explore the influence of control strategies or information topologies.
\end{itemize}

A key difficulty is that lateral safety is inherently spatial: because the road geometry and physical boundaries (e.g., lane edges and curbs) vary along the path, what matters is whether tracking errors attenuate as vehicles pass the same location on the road. This corresponds to a position-based (Eulerian) viewpoint. However, vehicles traverse the path asynchronously, so comparing errors at the same time instant typically compares different locations and curvatures, mixing controller performance with geometric difficulty. This motivates reparameterizing the dynamics by path arc length, expressing tracking errors (and control inputs) as functions of distance along the path. Fig.~\ref{fig: reparameterization} illustrates the difference between comparing lateral (cross-track) errors $e_{lat}$ at a common time $t_0$ versus a common arc length $l_0$, using a two-vehicle platoon as an example.
\begin{figure}[h]
\vspace{-0.5em}
    \centering
    \subfigure[time-based]{\includegraphics[width=0.42\textwidth]{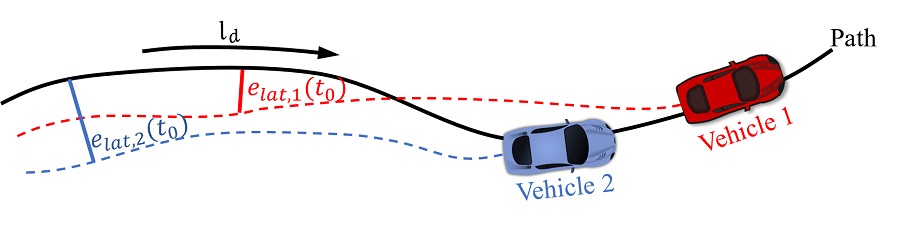}}
    \subfigure[arc-length-based]{\includegraphics[width=0.45\textwidth]{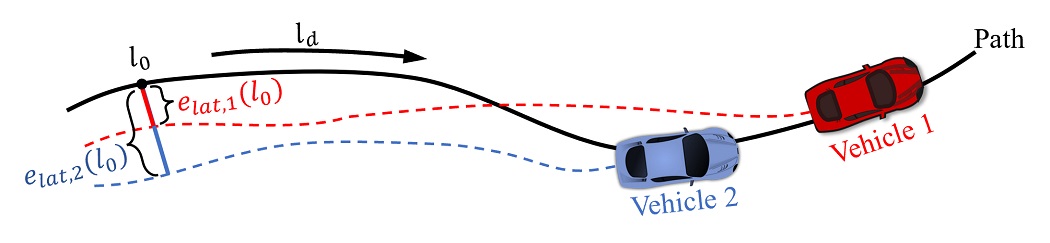}}
    \caption{Tracking error parameterizations}
    \label{fig: reparameterization}
    \vspace{-0.5em}
\end{figure}

This paper addresses the identified gaps through the following contributions:
\begin{enumerate}
    \item A formulation of the lateral control problem for vehicle platoons using arc-length reparameterization. 
    \item A novel \textit{learn-from-predecessor} control strategy that improves lateral string stability.
    \item A general necessary and sufficient condition for  $\mathcal{L}_2$ lateral string stability.
    \item A comprehensive analysis of different combinations of information acquisition modes, control strategies, and error measures, including impossibility results.
\end{enumerate}

The remainder of the paper is organized as follows. Section~\ref{sec2} introduces the platoon lateral control problem. Section~\ref{sec4} defines  $\mathcal{L}_2$ lateral string stability, states the main results, and summarizes the analysis. Section~\ref{sec6} validates the findings through simulation. Section~\ref{sec7} concludes the paper.

\section{Problem Formulation}
\label{sec2}

\subsection{Vehicle Platoon Lateral Control Problem}
We consider a finite homogeneous platoon of $m$ vehicles, each tasked with tracking a desired path by controlling its steering angle, as shown in Fig. \ref{fig: platoon_setting}. The availability of the desired path information may vary across vehicles.
\begin{figure}[htb!]
\vspace{-0.5em}
    \centering
    \includegraphics[width=0.9\linewidth]{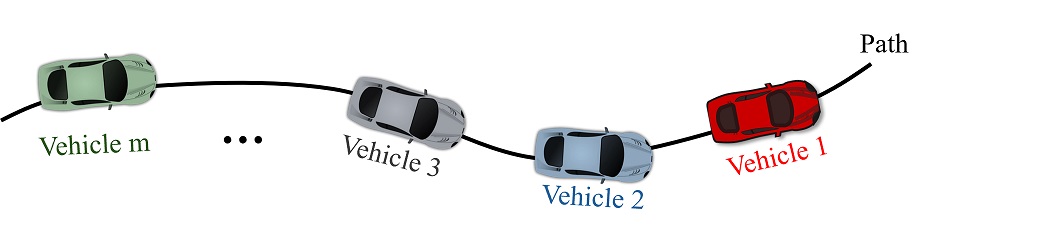}
    \caption{Vehicle platoon lateral control problem}
    \label{fig: platoon_setting}
    \vspace{-0.5em}
\end{figure}

The dynamics of each vehicle in the platoon are described by the tracking error dynamics model \cite{liu2020lateral,peng1990vehicle}. 
As shown in Fig.~\ref{fig: representation_heading_position_errors}, let $(X_v, Y_v)$ denote the vehicle's center of gravity (C.G.), and $(X_0, Y_0)$ its orthogonal projection onto the desired path. Define the lateral error $e_{lat}$ and heading error $\tilde{\theta} := \theta - \theta_R$, where $\theta$ denotes the vehicle's heading angle and $\theta_R$ the desired heading angle at $(X_0, Y_0)$. From kinematics:
\begin{equation}
\label{eq: d_theta_R}
\dot{\theta}_R = \frac{v_x}{R} = v_x \kappa,
\end{equation}
where $v_x$ denotes the longitudinal speed, $R$ is the instantaneous turning radius, and $\kappa = \frac{1}{R}$ is the path curvature.

\begin{figure}[htb!]
\vspace{-0.5em}
    \centering
    \includegraphics[width=0.4\linewidth]{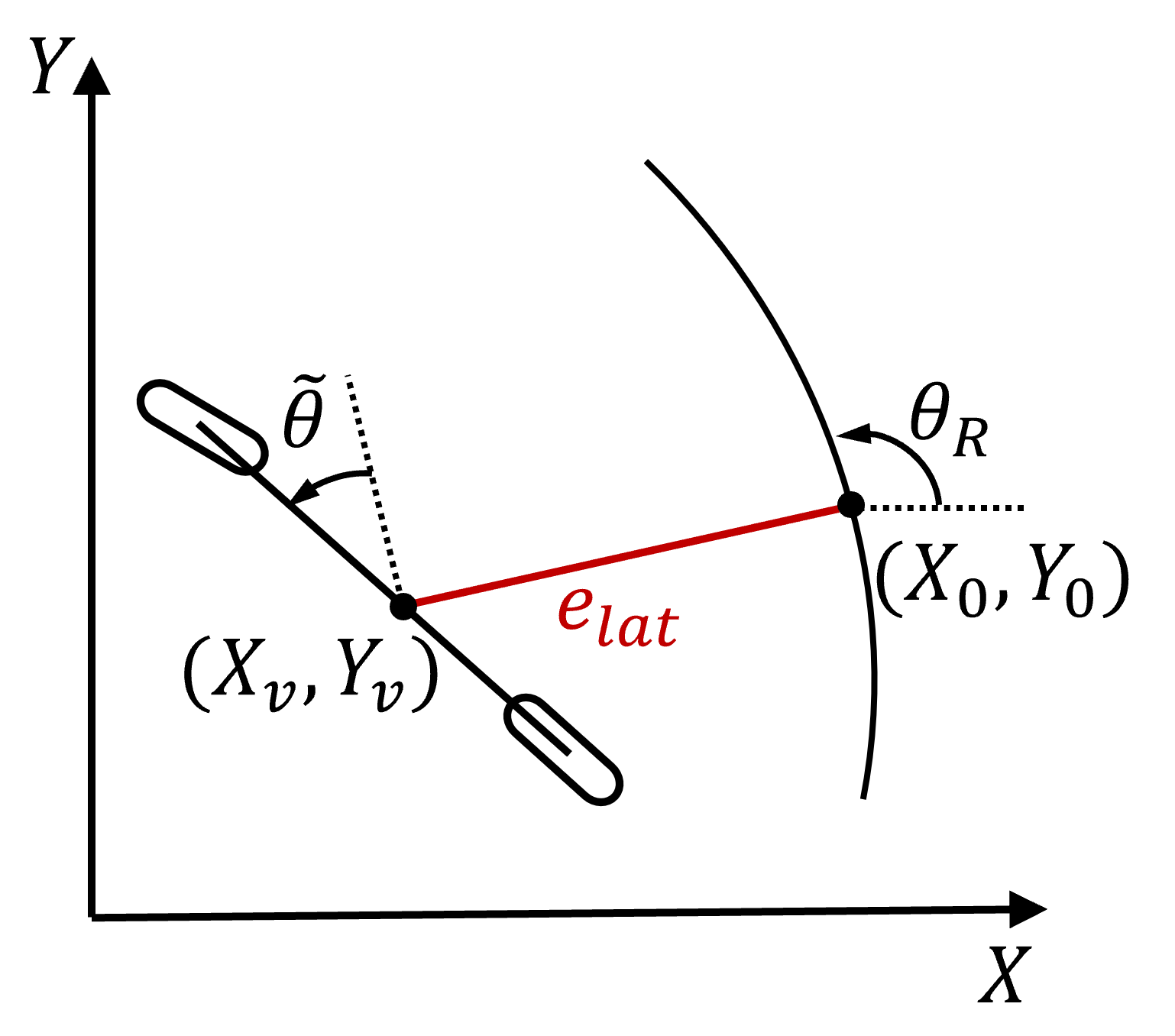}
    \caption{Heading and position error representation}
    \label{fig: representation_heading_position_errors}
    \vspace{-0.5em}
\end{figure}

Under standard assumptions \cite{liu2020lateral}—including constant longitudinal speed, small heading errors, negligible $\dot{R}$, large turning radius relative to $e_{lat}$, and omission of higher-order terms—the error dynamics of the $i^{\text{th}}$ vehicle are derived. Here, the superscript ``$des$'' denotes variables associated with the desired path, or, for error terms, quantities measured with respect to the desired path:
\begin{align}
\label{eq: governing}
\mathbf{M} \ddot{\mathbf{e}}_i^{des}(t) + \mathbf{C} \dot{\mathbf{e}}_i^{des}(t) + \mathbf{L} \mathbf{e}_i^{des}(t)&=\mathbf{B} u_i(t) - \tfrac{1}{v_x}\mathbf{F} \dot\theta^{des}(t), \notag\\ 
&~~\forall i=1,2,\ldots,m,
\end{align}
\[\text{where }
\mathbf{e}_i := \begin{bmatrix} e_{lat,i} \\ \tilde{ \theta}_i \end{bmatrix}, \quad
u_i := \delta_{f,i}, \quad
\mathbf{M} := \begin{bmatrix} m & 0 \\ 0 & I_z \end{bmatrix}, \quad \]
\vspace{-0.5em}
\[
\mathbf{C} := \begin{bmatrix} \frac{C_f + C_r}{v_x} & \frac{a C_f - b C_r}{v_x} \\ \frac{a C_f - b C_r}{v_x} & \frac{a^2 C_f + b^2 C_r}{v_x} \end{bmatrix}, \mathbf{L} := \begin{bmatrix} 0 & -(C_f + C_r) \\ 0 & -(a C_f - b C_r) \end{bmatrix}, 
\]
\vspace{-1em}
\begin{align*}
\mathbf{B} := \begin{bmatrix} C_f \\ a C_f \end{bmatrix}, \quad
\mathbf{F} := \begin{bmatrix} mv_x^2 + a C_f - b C_r \\ a^2 C_f + b^2 C_r. \end{bmatrix}.
\end{align*}
Here, $\delta_{f,i}$ denotes the front steering angle, and $a$, $b$ denote the distances from the C.G. to the front and rear axles, respectively. The parameters $C_f$ and $C_r$ are the cornering stiffnesses of the front and rear tires, respectively. The parameter $m$ is the vehicle mass and $I_z$ is the moment of inertia about the vertical axis.

To analyze lateral string stability, we must compare tracking errors across vehicles, which requires evaluating errors at the same spatial location along the path. To this end, we reparameterize the above model using the arc length of the desired path, denoted $l_d$. Note that we use the prime notation for arc-length derivatives (i.e., $x' = \tfrac{dx}{dl}$):
\begin{align}
\label{eq: i_th_governing_reparam}
&v_x^2 \mathbf{M} \left( \mathbf{e}_i^{des}(l_d) \right)'' + v_x \mathbf{C} \left( \mathbf{e}_i^{des}(l_d) \right)' + \mathbf{L} \mathbf{e}_i^{des}(l_d) \notag \\
=& \mathbf{B} u_i(l_d) - \mathbf{F} \left(\theta^{des}(l_d)\right)',~~\forall i=1,2,\ldots,m.
\end{align}

Fig. \ref{fig: reparameterization} illustrates the difference between time parameterization and arc-length parameterization.

\subsection{Information Acquisition Modes and Control Strategies}
We distinguish two modes of information acquisition for the vehicles: onboard sensing and V2V communication. Each mode entails a different control strategy, reflecting the distinct availability of information.

For the onboard sensing mode, the lead vehicle directly tracks the desired path, while each following vehicle records its predecessor's traveled path, represented by X, Y coordinates and heading angle, as shown in Fig. \ref{fig: information flow_predecessor}. For example, a following vehicle can estimate the predecessor’s relative position and heading from onboard sensors (e.g., stereo vision, LiDAR-based tracking, or camera–LiDAR fusion), and then transform them into global X, Y, and heading using the following vehicle's own GNSS-based pose estimate. We refer to this recorded path as the reference path, and each following vehicle tracks its respective reference path.
\vspace{-0.5em}
\begin{figure}[h]
    \centering
    {\includegraphics[width=0.4\textwidth]{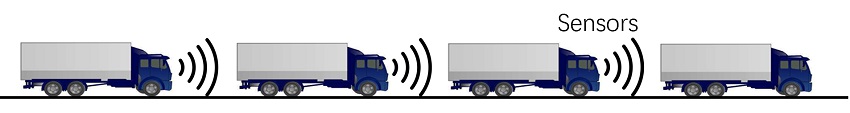}}
    \caption{Onboard sensing}
    \label{fig: information flow_predecessor}
\end{figure}
\vspace{-0.5em}

A feedback–feedforward (FF) strategy, which is widely used for single vehicle lateral control \cite{liu2020lateral}, is employed for this mode. We use superscript ``$ref$'' to denote variables associated with the reference path, or, for error terms, quantities measured with respect to the reference path. Under the assumption that each vehicle has small heading errors relative to its reference path, we obtain the following relations, with Fig. \ref{fig: ref_error_pred} presenting an exemplary illustration. Note that only $e_{lat,i}^{ref}$ is shown in the figure and $\tilde{\theta}_i^{ref}$ is omitted for clarity.
\begin{align}
    \mathbf{e}_i^{ref}(l_d)&\approx\mathbf{e}_i^{des}(l_d)- \mathbf{e}_{i-1}^{des}(l_d), \quad\text{for } i=2,3,\ldots,m. \label{eq: e_ref_relation} \\
    \theta^{ref}_i(l_d)&\approx\theta^{des}(l_d)+\tilde{\theta}^{des}_{i-1}(l_d) \notag \\
    &=\theta^{des}(l_d)+[0~ 1]\mathbf{e}_{i-1}^{des}(l_d),\quad\text{for }i=2,3,\ldots,m. \label{eq: theta_ref_relation}
\end{align}
\begin{figure}[h]
\vspace{-0.5em}
    \centering
    {\includegraphics[width=0.47\textwidth]{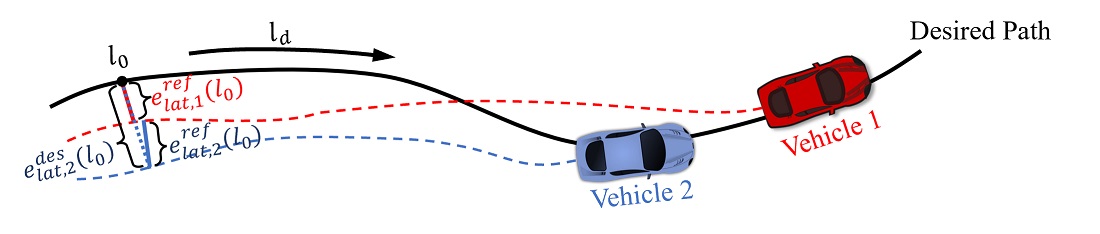}}
    \caption{Reference error computation}
    \label{fig: ref_error_pred}
    \vspace{-0.5em}
\end{figure}

Using the above relations, the arc-length reparameterized control law of the FF strategy is given as follows:
\begin{equation}
\label{eq: i_th_u_pred_track}
\small
u_i(l_d) =
\begin{cases}
\begin{aligned}
 &-\mathbf{K_P}\,\mathbf{e}_i^{des}(l_d)
 - v_x \mathbf{K_D} \big(\mathbf{e}_i^{des}(l_d)\big)' \\
 &~~ + k_{ff}\,\big(\theta^{des}(l_d)\big)' , \quad\quad\quad\quad\quad\quad\quad~ i=1,
\end{aligned}\\[5ex]
\begin{aligned}
 &-\mathbf{K_P} \left(\mathbf{e}_i^{des}(l_d)- \mathbf{e}_{i-1}^{des}(l_d)\right)  \\
 &~~ - v_x \mathbf{K_D} \left(\mathbf{e}_i^{des}(l_d)- \mathbf{e}_{i-1}^{des}(l_d) \right)'
  \\
 &~~+ k_{ff}  \left(\theta^{des}(l_d)+[0~ 1]\mathbf{e}_{i-1}^{des}(l_d)\right)',~i=2,\ldots,m,
\end{aligned}
\end{cases}
\end{equation}
where the feedback gain vectors $\mathbf{K_P} := [k_{e_{lat}} \quad k_{\tilde{\theta}}]$ and $\mathbf{K_D} := [k_{\dot{e}_{lat}} \quad k_{\dot{\tilde{\theta}}}]$, while $k_{ff}$ is a scalar feedforward gain. Note that the above equation is useful for analyzing the propagation of $\mathbf{e}_i^{des}$. In controller implementation, $\mathbf{e}_i^{ref}$ and $\theta_i^{ref}$ are used for feedback and feedforward computation.

For the V2V mode, each following vehicle communicates with its predecessor, as shown in Fig. \ref{fig: information flow_desired}.
\begin{figure}[h]
\vspace{-0.5em}
    \centering
    {\includegraphics[width=0.4\textwidth]{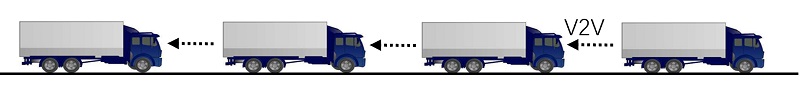}}
    \caption{V2V communication}
    \label{fig: information flow_desired}
    \vspace{-0.5em}
\end{figure}

We propose a novel learn-from-predecessor (LFP) strategy for the V2V mode to improve lateral string stability. Its implementation is given by:
\begin{align}
&u_i(l_d)= 
-\mathbf{K_P} \mathbf{e}_i^{des}(l_d) - v_x \mathbf{K_D} \left( \mathbf{e}_i^{des}(l_d) \right)'
+ u_{l,i}(l_d), \notag 
\end{align}
where
\begin{equation}
\label{eq: ith_veh_control_learn}
u_{l,i}(l_d) =
\begin{cases}
  k_{ff}\,\big(\theta^{des}(l_d)\big)', & i=1, \\[2ex]
  u_{l,i-1}(l_d)+\mathbf{K_{LP}}\,\mathbf{y}_{i-1}(l_d) 
  \\~+ \mathbf{K_{LD}}\,\mathbf{y}'_{i-1}(l_d), & i=2,\ldots,m.
\end{cases}
\end{equation}
Here, $u_{l,i}(l_d)$ denotes the learned term of the $i^\text{th}$ vehicle. Each following vehicle ($i= 2,\ldots,m$) updates its predecessor's learned term $u_{l,i-1}(l_d)$ with two additional terms that depend on its predecessor’s output $\mathbf{y}_{i-1}(l_d)$ and its derivative ${\mathbf{y}}'_{i-1}(l_d)$. The output is defined as $\mathbf{y}_{i-1}=\mathbf{C_{out}}\mathbf{e}_{i-1}^{des}$ where $\mathbf{C_{out}}$ is a selection matrix that extracts relevant components from the tracking error vector; this will be discussed in more detail later. The gains $\mathbf{K_{LP}}$ and $\mathbf{K_{LD}}$ represent the proportional and derivative learning gains, respectively.

With this strategy, each vehicle shares the desired path, its output $\mathbf{y}_{i-1}(l_d)$, and learned term $u_{l,i-1}$ with its follower $i$.
\begin{rem}
The LFP strategy is naturally robust to moderate communication delays since the spatial separation between vehicles already introduces a time gap. 
\end{rem}
\begin{rem}
  While the tracking-error model is derived using the orthogonal projection of the C.G. onto a path, in implementation we use the closest point on the path instead. For locally circular paths this coincides with the orthogonal projection, and it is standard practice for general paths.
\end{rem}

\section{Lateral string stability}
\label{sec4}

\subsection{Definition and Main Results}
\label{sec 4.1}
In this paper, we consider the following definition of lateral string stability.
\begin{defn}[$\mathcal{L}_2$ lateral string stability]
Consider the platoon system in \eqref{eq: i_th_governing_reparam}, if for some $\mathbf{C_{out}}$ and $\mathbf{y}_i=\mathbf{C_{out}}\mathbf{e}_i^{des}$, it holds that
\begin{equation}
\label{eq:  Lateral string stability condition}
\| \mathbf{y}_i(l_d)  \|_{\mathcal{L}_2} < \| \mathbf{y}_{i-1}(l_d)  \|_{\mathcal{L}_2}, \quad \forall i=2,3,\ldots, m,
\end{equation}
then the system is $\mathcal{L}_2$ lateral string stable with respect to the output $\mathbf{y}_i=\mathbf{C_{out}}\mathbf{e}_i^{des}$.
\end{defn}

\begin{rem}
\label{rem: output choices}
We consider two outputs for $\mathbf{y}_i$. The first is
\[
\mathbf{y}_i = \mathbf{e}_i^{des} =
\begin{bmatrix}
e_{lat,i}^{des} \\[2pt]
\tilde{\theta}_i^{des}
\end{bmatrix},
\qquad
\mathbf{C}_{out} = \mathbf{I}_2,
\]
which captures both lateral tracking errors.

The second uses only the lateral error:
\[
\mathbf{y}_i = e_{lat,i}^{des},
\qquad
\mathbf{C}_{out} = [\,1\;\;0\,],
\]
a choice motivated by its direct safety relevance as it quantifies lateral deviation from the desired path.
\end{rem}

For linear time-invariant (LTI) systems, frequency-domain analysis is particularly convenient. In particular, Parseval’s theorem relates the induced $\mathcal{L}_2$ norm with the system's frequency response \cite{bhattacharyya2018linear}. For such an analysis, we utilize the Laplace transform. Note that \eqref{eq: i_th_governing_reparam} captures the plant dynamics in response to the control input $u_i$, and its structure remains the same across different information acquisition modes and control strategies. For notational simplicity, we adopt lowercase symbols to denote Laplace-transformed variables; for instance, the Laplace transform of $z(l_d)$ is denoted by $z(s)$, where $s$ is the Laplace variable. Assuming zero initial conditions, the Laplace-domain representation of \eqref{eq: i_th_governing_reparam} is:
\begin{equation}
\label{model Laplace}
   (s^2v_x^2\mathbf{M} + sv_x\mathbf{C}  + \mathbf{L})\mathbf{e}^{des}_i(s) = \mathbf{B} {u_i}(s) - s\mathbf{F}\theta^{des}(s),
\end{equation}
the expressions for different $u_i(s)$ will be provided later.

We establish some general results that will simplify and unify several arguments to come.
We begin with a necessary and sufficient condition for  $\mathcal{L}_2$ lateral string stability when the vehicle interconnection can be expressed as input–output maps between $\mathbf{y}_{i-1}$ and $\mathbf{y}_i$, stated in the following theorem.

\begin{thm}
\label{prop: LFP DT frequency-domain condition}
Consider the platoon system in \eqref{eq: i_th_governing_reparam}, and suppose the interconnections between subsystems can be expressed in the Laplace domain as
\begin{equation}
\mathbf{y}_i(s) = \mathbf{H}_i(s)\,\mathbf{y}_{i-1}(s), \quad \forall i=2,3,\ldots,m,
\end{equation}
where $\mathbf{H}_i(s)$ denotes the transfer function from $\mathbf{y}_{i-1}(s)$ to $\mathbf{y}_i(s)$. The system is $\mathcal{L}_2$ lateral string stable with respect to the output $\mathbf{y}_i$ if and only if
\begin{equation}
\|\mathbf{H}_i(s)\|_{\mathcal{H}_\infty} := \sup_{\omega\in\mathbb{R}} \sigma_1\!\left(\mathbf{H}_i(j\omega)\right) < 1, \quad \forall i=2,3,\ldots,m.
\end{equation}
\end{thm}

\begin{proof}
The proof follows from the known fact that for causal, stable LTI systems, the $\mathcal{H}_\infty$ norm equals the induced $\mathcal{L}_2$ norm \cite{bhattacharyya2018linear}.

\end{proof}
\begin{rem}
When $\mathbf{y}_i=e_{lat,i}^{des}$, $\mathbf{H}_i(s)$ reduces to a scalar transfer function $H_i(s)$ rather than a matrix. In this case, the $\mathcal{H}_\infty$ norm definition simplifies to the standard SISO form
\begin{equation}
\|H_i(s)\|_{\mathcal{H}_\infty} := \sup_{\omega \in \mathbb{R}} \left| H_i(j\omega) \right|.
\end{equation}
\end{rem}

 Next, we present a useful lemma that provides a perturbation bound on singular values. Note that $\mathbb{M}_{n,m}$ denotes the set of all complex $n\times m$ matrices.
\begin{lem}[Corollary 7.3.5 in \cite{horn2012matrix}]
\label{lem: singular value bound}
Let $\mathbf{P}, \mathbf{Q} \in \mathbb{M}_{n,m}$ and let $q = \min\{m, n\}$. Let 
$\sigma_1(\mathbf{P}) \geq \cdots \geq \sigma_q(\mathbf{P})$ and 
$\sigma_1(\mathbf{Q}) \geq \cdots \geq \sigma_q(\mathbf{Q})$ be the nonincreasingly ordered singular values of $\mathbf{P}$ and $\mathbf{Q}$, respectively. Then
\begin{equation}
|\sigma_j(\mathbf{P}) - \sigma_j(\mathbf{Q})| \leq \sigma_1\left(\mathbf{P} - \mathbf{Q}\right) ,~ \text{for each }j = 1, \ldots, q.
\end{equation}
\end{lem}
Using Lemma~\ref{lem: singular value bound}, we establish the following theorem, which is useful for showing that  $\mathcal{L}_2$ lateral string stability cannot be achieved with respect to the output $\mathbf{y}_i = \mathbf{e}_i^{\text{des}}$.
\begin{thm}
\label{prop: multi-error not stable general}
    Suppose there exists an $i \in \{2, \dots, m\}$ such that we have the input-output relationship
    \begin{equation}
\mathbf{e}_{i}^{des}(s) = \mathbf{H}_i(s)\, \mathbf{e}_{i-1}^{des}(s),
\end{equation}
where $\mathbf{H}_i(s) = \mathbf{I}_2 + \mathbf{R}(s)$, with $\mathbf{I}_2$ the $2 \times 2$ identity matrix and $\mathbf{R}(s)$ rank-deficient for all $s$. Then the platoon cannot be $\mathcal{L}_2$ lateral string stable with respect to the output $\mathbf{y}_i = \mathbf{e}_i^{des}$.
\end{thm}
\begin{proof}
    Substituting $s=j\omega$, we obtain the frequency-domain representation of the input-output relationship:
\begin{align}
\mathbf{e}^{des}_{i}(j\omega) = \mathbf{H}_i(j\omega)\mathbf{e}^{des}_{i-1}(j\omega).
\end{align}
Since $\mathbf{R}(s)$ is rank-deficient for all $s$, we have
\begin{equation}
    \sigma_2\left(\mathbf{R}(j\omega)\right)=0,~~~~\forall\omega\in\mathbb{R}.
\end{equation}
From Lemma \ref{lem: singular value bound}, we have
\begin{align}
    \sigma_1\left(\mathbf{H}_i(j\omega)\right)\geq\left|\sigma_2\left(\mathbf{I}_2\right)-\sigma_2\left(-\mathbf{R}(j\omega)\right)\right|=1, ~~~~\forall\omega\in\mathbb{R}.
\end{align}
Therefore, the $\mathcal{H}_\infty$ norm of $\mathbf{H}_i(s)$ satisfies 
\begin{equation}
    \|\mathbf{H}_i(s)\|_{\mathcal{H}_\infty}:=\sup_{\omega\in\mathbb{R}}\sigma_1\left(\mathbf{H}_i(j\omega)\right)\geq1.
\end{equation}
Using Theorem \ref{prop: LFP DT frequency-domain condition}, the proof is complete.
\end{proof}

\subsection{Summary of Analysis Results}
\label{sec5.1}
In this subsection, we summarize the results of the detailed analysis and provide proof sketches due to space limitations.

Recall that for the FF strategy with onboard sensing, the control law is given in \eqref{eq: i_th_u_pred_track}. Applying the Laplace transform to both sides of it, we obtain 
\begin{align}
\label{eq: u Laplace FF PT}
u_i(s) =
\begin{cases}
-(\mathbf{K_P} + s v_x \mathbf{K_D})\, \mathbf{e}_i^{des}(s) + s k_{ff} \theta^{des}(s),  \\\qquad\qquad\qquad\qquad\qquad\quad\text{for } i = 1, \\ \\
-(\mathbf{K_P} + s v_x \mathbf{K_D}) 
\left( \mathbf{e}_i^{des}(s) - \mathbf{e}_{i-1}^{des}(s) \right)
\\~+ s k_{ff} \left( \theta^{des}(s) + [0~~1] \mathbf{e}_{i-1}^{des}(s) \right),  \\\qquad\qquad\qquad\qquad\qquad\quad\text{for } i = 2,3,\ldots,m.
\end{cases} 
\end{align}

For the LFP strategy with V2V, the control law is given in \eqref{eq: ith_veh_control_learn}. Applying the Laplace transform to both sides of it,   
\begin{align}
&u_i(s)= 
-(\mathbf{K_P} + sv_x \mathbf{K_D}) \mathbf{e}_i^{des}(s)
+ u_{l,i}(s), \notag 
\end{align}
where
\begin{align}
\label{LFP V2V}
&u_{l,i}(s)=\begin{cases}
     sk_{ff} \theta^{des}(s)&\text{}i=1, \\ 
     \\
u_{l,i-1}(s)\\
~+(\mathbf{K_{LP}}+s\mathbf{K_{LD}})\mathbf{y}_{i-1}(s)&\text{}i=2,\ldots,m. 
\end{cases}
\end{align}

We can obtain the following impossibility results, that is, the following \textit{cannot} achieve  $\mathcal{L}_2$ lateral string stability:
\begin{enumerate}
  \item FF strategy with onboard sensing, w.r.t. $\mathbf{y}_i=\mathbf{e}_i^{des}$.
  \item FF strategy with onboard sensing, w.r.t. $\mathbf{y}_i=e_{lat,i}^{des}$.
  \item LFP strategy with V2V, w.r.t. $\mathbf{y}_i=\mathbf{e}_i^{des}$.
\end{enumerate}
In the above results, case 1) and case 3) follow by substituting 
\eqref{eq: u Laplace FF PT} and \eqref{LFP V2V}, respectively, into 
\eqref{model Laplace}, rearranging, and applying Theorem~\ref{prop: multi-error not stable general}. 
Case 2) is shown by a counterexample: when the lead vehicle’s lateral error is identically zero, 
 a nonzero steady-state heading error arises 
which in turn implies that the follower’s lateral error cannot remain zero.

In fact, the only case achieving  $\mathcal{L}_2$ lateral string stability 
is the LFP strategy with V2V and the output choice $\mathbf{y}_i=e_{lat,i}^{des}$. 
This underscores the role of V2V communication in enhancing lateral string stability.

\section{Numerical simulations}
\label{sec6}
To validate the proposed theoretical findings, we perform numerical simulation tests using parameters based on real-world data. A platoon of 6 vehicles is considered, with the dynamics of each vehicle described by a bicycle model. The parameters of the bicycle model are identified from a Lincoln MKZ \cite{liu2020lateral}, as listed in Table~\ref{table: parameters}. The platoon operates on a test track at $v_x=10m/s$ with the goal of tracking a desired path. The test track and corresponding desired path are shown in Fig.~\ref{fig: track}, where the grey area denotes the track, the blue line indicates the desired path, and the blue box marks the location where $l_d=0$. The desired path consists of four lane changes over the course of a full circuit, with absolute turning radii ranging from $7.4m$ to $2.9\times 10^4m$, thereby capturing road characteristics of both highway and urban scenarios.
\begin{table}[ht]
\vspace{-0.5em}
\centering
\footnotesize
\caption{Parameters of the Lincoln MKZ}
\begin{tabular}{cccc}
\hline
\hline
\textbf{Parameter} & \textbf{Symbol} & \textbf{Unit} & \textbf{Value}  \\
\hline
Mass & $m$ & kg & 1896 \\
Moment of Inertia & $I_z$ & kg$\cdot$m$^2$ & 3803  \\
Front Cornering Stiffness & $C_{f}$ & N/rad & 400000  \\
Rear Cornering Stiffness & $C_{r}$ & N/rad & 381900  \\
Distance from C.G. to front axle & $a$ & m & 1.2682   \\
Distance from C.G. to rear axle & $b$ &m & 1.5818   \\
\hline
\hline
\label{table: parameters}
\vspace{-0.5em}
\end{tabular}
\end{table}
\begin{figure}[htb!]
\vspace{-0.5em}
    \centering
    \includegraphics[width = 0.6\linewidth]{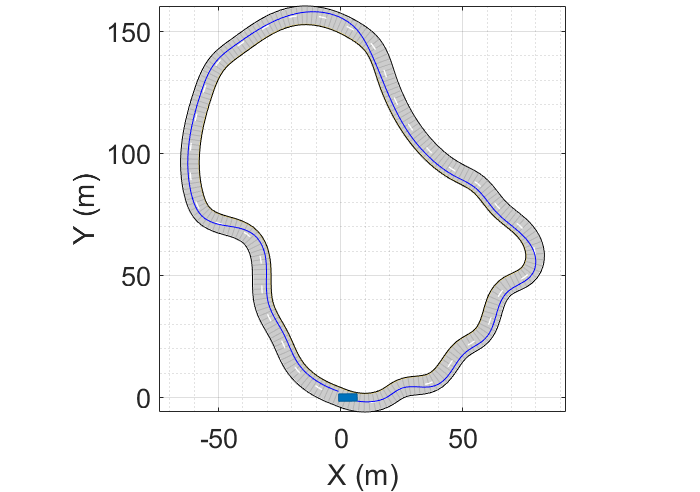}
    \caption{Test track and desired path}
    \label{fig: track}
    \vspace{-0.5em}
\end{figure}

Using the setup described above, we validate the proposed theoretical findings. Specifically, in Subsection \ref{sec 6.1}, we design the LFP strategy with V2V, and validate that it is $\mathcal{L}_2$ lateral string stable with respect to the output $\mathbf{y}_i=e_{lat,i}^{des}$. In Subsection \ref{sec 6.3}, we modify the LFP strategy to the FF strategy with only onboard sensing, and validate its instability results.

\subsection{LFP Strategy with V2V}
\label{sec 6.1}
For the LFP strategy with V2V, we design the strategy so that the conditions given by Theorem \ref{prop: LFP DT frequency-domain condition} are met with respect to the output $\mathbf{y}_i = e_{lat,i}^{des}$.  
We leverage existing literature by adopting the stabilizing feedback gains $\mathbf{K_P} = [k_{e_{lat}} \quad k_{\tilde{\theta}}]$ and $\mathbf{K_D} = [k_{\dot{e}_{lat}} \quad k_{\dot{\tilde{\theta}}}]$ from~\cite{liu2020lateral}, which were designed for the same Lincoln MKZ vehicle considered in this paper.  
The feedforward gain $k_{ff}$ is adopted from~\cite{liu2020lateral}, which ensures zero steady-state lateral error when tracking a circular arc, and is computed as $k_{ff} = a + b + \frac{m v_x^2}{a+b} \left( \tfrac{b}{C_f} - \tfrac{a}{C_r} + \tfrac{a}{C_r} k_{\tilde{\theta}} \right) - b k_{\tilde{\theta}}$.  
The learning gains $\mathbf{K_{LP}}$ and $\mathbf{K_{LD}}$ (scalars for $\mathbf{y}_i = e_{lat,i}^{des}$) are selected according to Theorem \ref{prop: LFP DT frequency-domain condition}.  
The designed numerical values of all controller gains are listed in Table~\ref{tab: controller gains design}.
\begin{table}[h]
\vspace{-0.5em}
\centering
\footnotesize
\caption{Controller gains for the designed LFP strategy}
\begin{tabular}{c|c|c|c|c|c|c}
\hline \hline
$k_{e_{lat}}$ & $k_{\tilde{\theta}}$ & $k_{\dot e_{lat}}$ & $k_{\dot{\tilde{\theta}}}$ & $k_{ff}$ & $\mathbf{K_{LP}}$ & $\mathbf{K_{LD}}$\\[1.0001ex]
\hline
0.06 & 0.96 & 0  & 0.08 & 1.59 & -0.04  & -0.3 \\
\hline \hline
\end{tabular}
\label{tab: controller gains design}
\vspace{-0.5em}
\end{table}

Figure \ref{fig: LFP DT results} presents (a) the traveled paths in the \(X\)–\(Y\) plane, (b) the lateral error \(e_{lat}^{des}\) versus arc length \(l_d\), and (c) the heading error \(\tilde{\theta}^{des}\) versus \(l_d\) for all vehicles. As seen in Fig. \ref{fig: LFP DT results}(b) and in the zoomed-in view of Fig. \ref{fig: LFP DT results}(a), the lateral error \(e_{lat}^{des}\) consistently diminishes from one vehicle to the next, providing clear visual confirmation of  lateral string stability with respect to \(\mathbf{y}_i = e_{lat,i}^{des}\). In contrast, Fig. \ref{fig: LFP DT results}(c) reveals no consistent attenuation pattern in the heading error \(\tilde{\theta}^{des}\).

\begin{figure}[h]
\vspace{-0.5em}
    \centering
    \subfigure[traveled paths in $X-Y$ plane]{\includegraphics[width=0.239\textwidth]{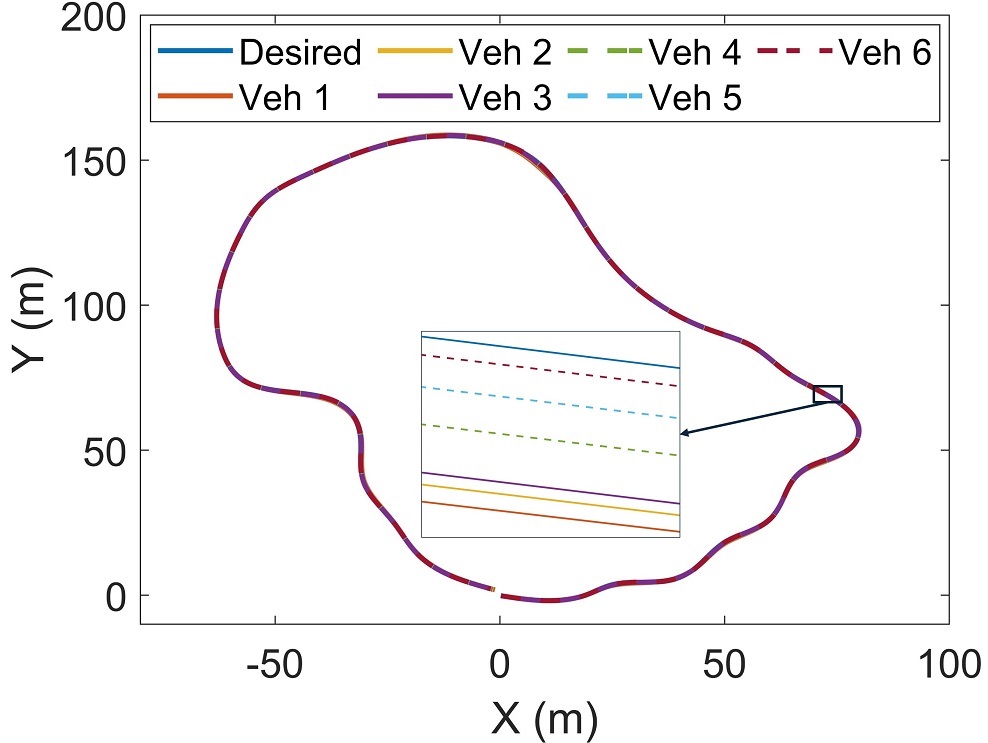}} \\
    \subfigure[lateral error v.s. arc length]{\includegraphics[width=0.239\textwidth]{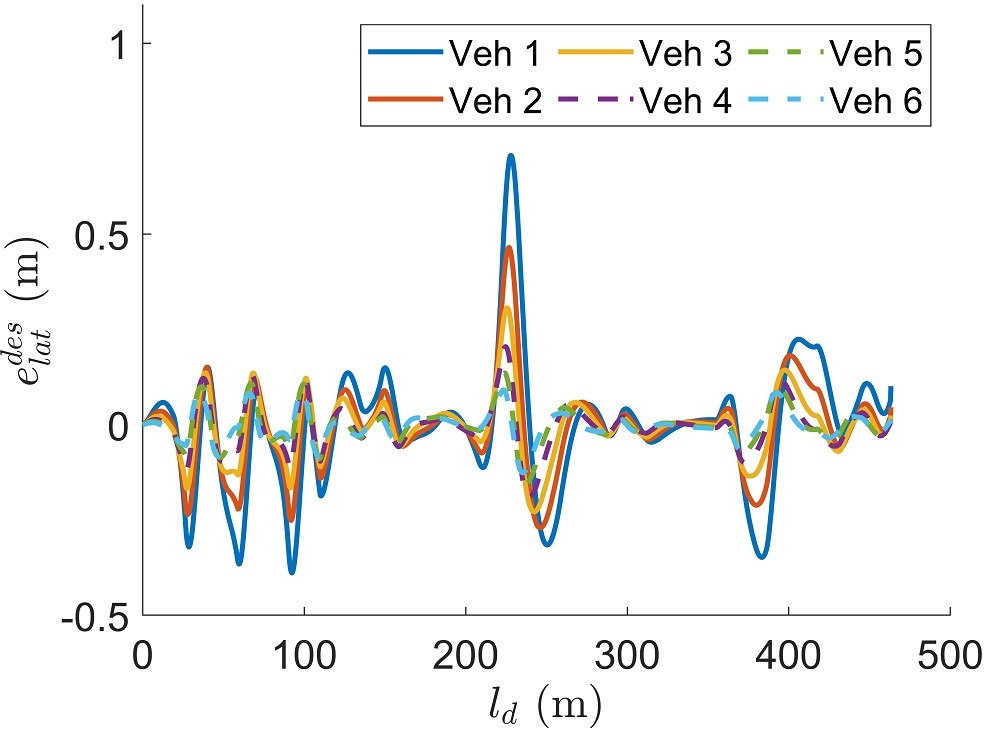}}
    \subfigure[heading error v.s. arc length]
{\includegraphics[width=0.239\textwidth]{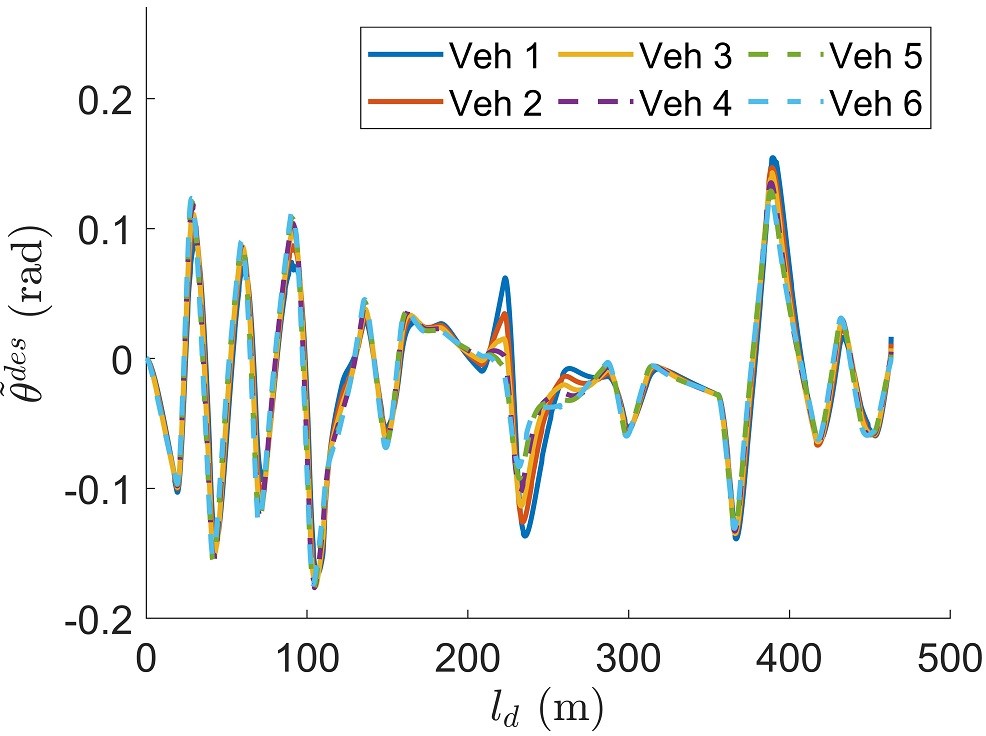}}
    \caption{Simulation results for LFP with V2V}
    \label{fig: LFP DT results}
    \end{figure}

A more quantitative comparison is given in Fig.~\ref{fig: LFP DT quantitative results}, which plots the $\mathcal{L}_2$ norms of the two outputs of interest: $e_{lat}^{des}$ in Fig.~\ref{fig: LFP DT quantitative results}(a) and $\mathbf{e}^{des}$ in Fig.~\ref{fig: LFP DT quantitative results}(b). The $\mathcal{L}_2$ norms are obtained by approximating the continuous-time integrals with discrete-time summations over the simulation horizon. To better highlight the pattern, the platoon length is extended to 12 vehicles for the quantitative results. As seen in Fig.~\ref{fig: LFP DT quantitative results}(a), the lateral error norm decreases from one vehicle to the next, confirming the  $\mathcal{L}_2$ lateral string stability of the LFP strategy with V2V with respect to $\mathbf{y}_i = e_{lat,i}^{des}$. Notably, Fig.~\ref{fig: LFP DT quantitative results}(b) shows that, although  $\mathcal{L}_2$ lateral string stability with respect to $\mathbf{y}_i = \mathbf{e}_i^{des}$ is not guaranteed for all scenarios, it still holds empirically in this particular test scenario.
\begin{figure}[h]
    \centering
    \subfigure[$\mathcal{L}_2$ norm of $e_{lat}^{des}$]{\includegraphics[width=0.235\textwidth]{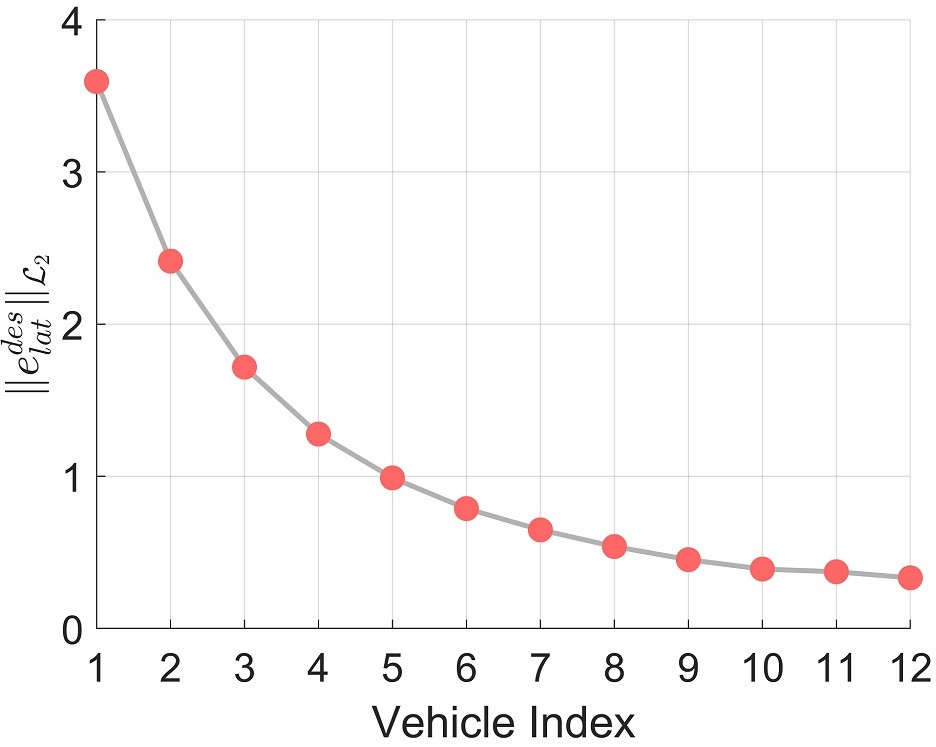}} 
    \subfigure[$\mathcal{L}_2$ norm of $\mathbf{e}^{des}$]{\includegraphics[width=0.243\textwidth]{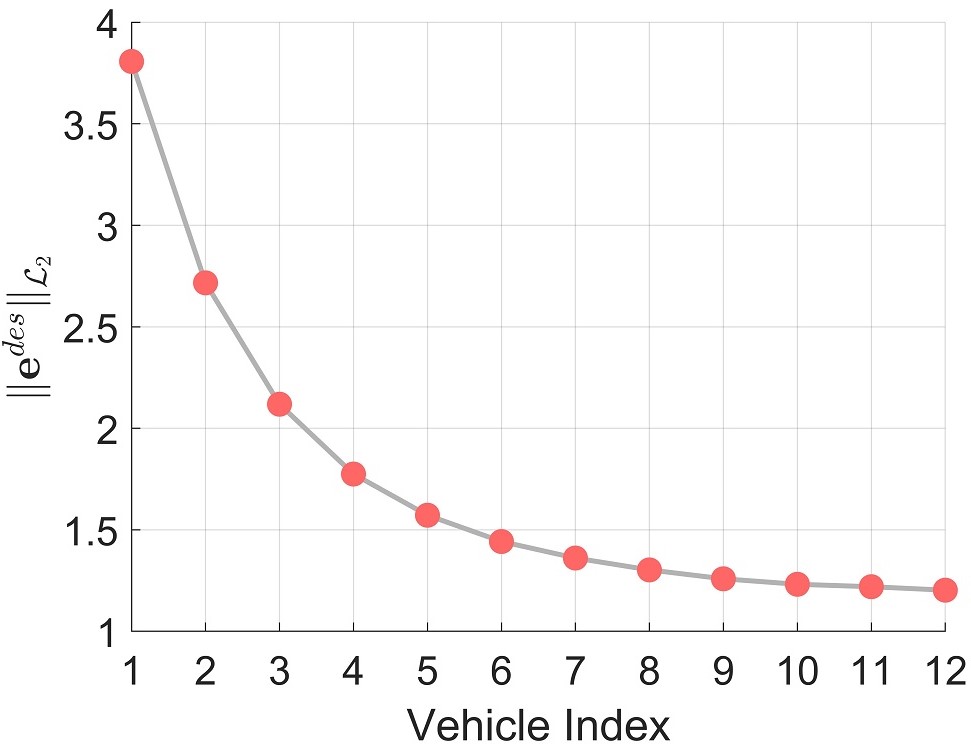}}
    \caption{Quantitative results for LFP with V2V}
    \label{fig: LFP DT quantitative results}
    \vspace{-0.5em}
    \end{figure}

\subsection{FF Strategy with Onboard Sensing}
\label{sec 6.3}
In this subsection, we validate the instability results of the FF strategy with onboard sensing. The control strategy switches from LFP to FF, while the values of the common controller gains $k_{e_{lat}}, k_{\tilde{\theta}}, k_{\dot e_{lat}}, k_{\dot{\tilde{\theta}}}$, and $k_{ff}$ are kept identical to those in Table~\ref{tab: controller gains design}. The tracking errors used for control input computation of each following vehicle are calculated relative to the recorded path traveled by its immediate predecessor.

Figure \ref{fig: FF results} shows (a) the traveled paths in the $X$–$Y$ plane, (b) the lateral error $e_{lat}^{des}$ v.s. arc length $l_d$, and (c) the heading error $\tilde{\theta}^{des}$ v.s. $l_d$. Across all subfigures, a pronounced amplification of errors along the platoon is evident.
\begin{figure}[h]
\vspace{-0.5em}
    \centering
    \subfigure[traveled paths in $X-Y$ plane]{\includegraphics[width=0.239\textwidth]{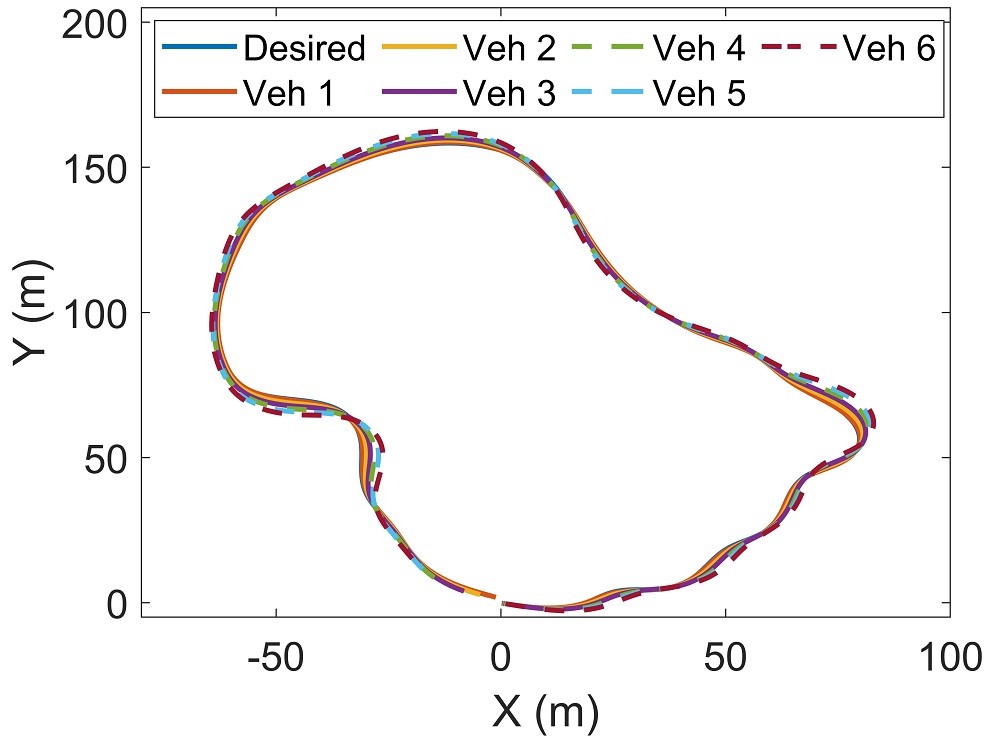}} \\
    \subfigure[lateral error v.s. arc length]{\includegraphics[width=0.239\textwidth]{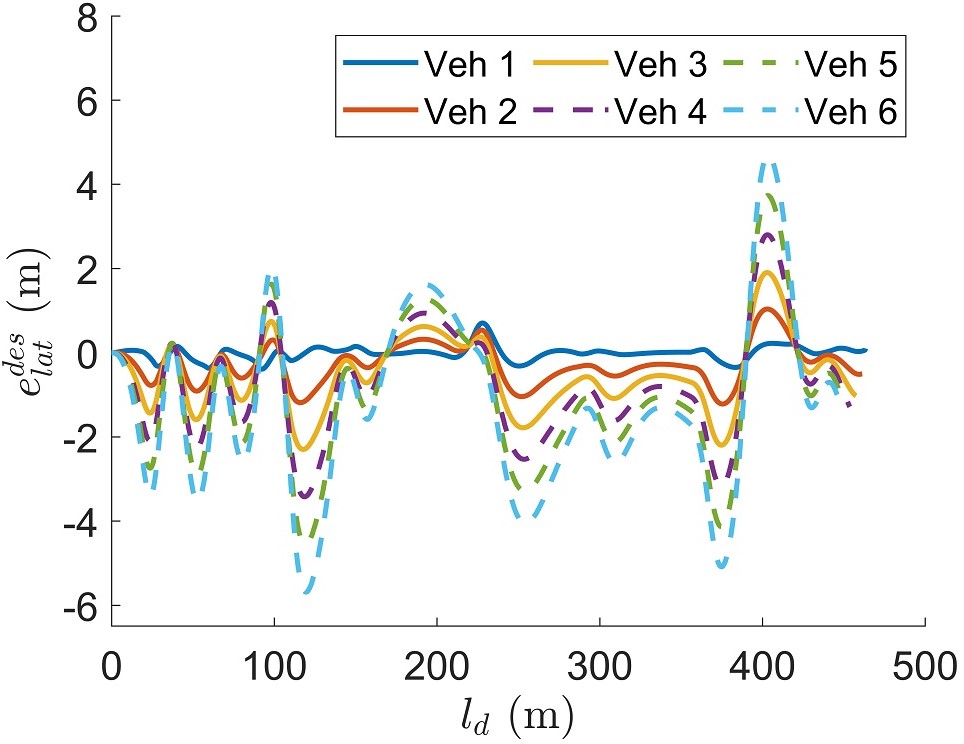}}
    \subfigure[heading error v.s. arc length]
{\includegraphics[width=0.239\textwidth]{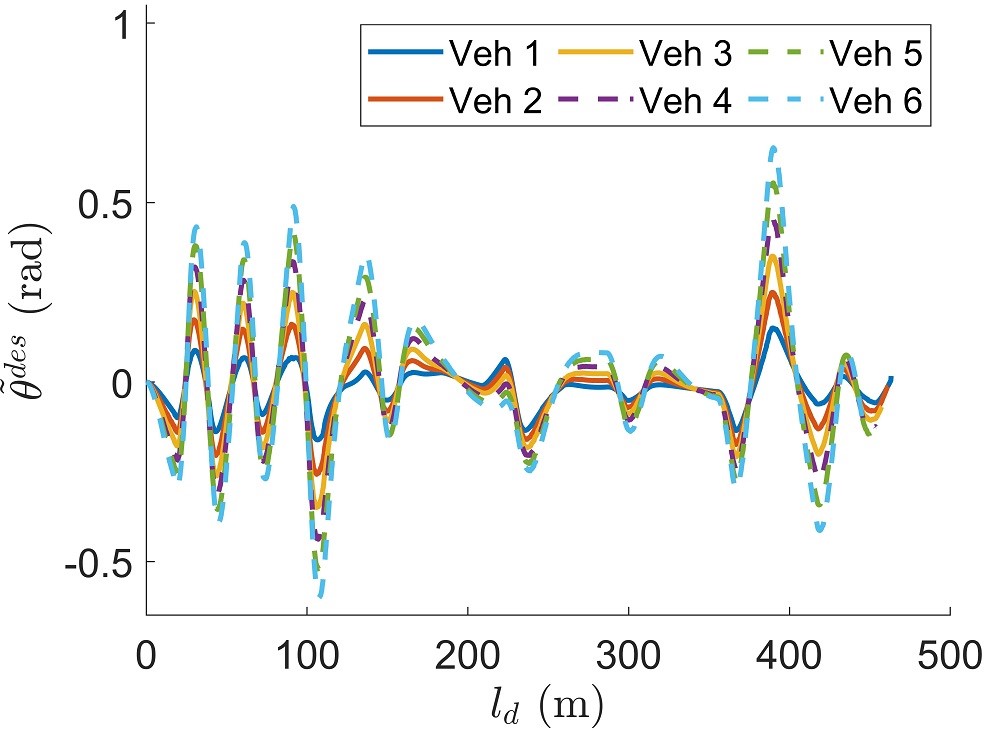}}
    \caption{Simulation results for FF with onboard sensing}
    \label{fig: FF results}
    \end{figure}

Similarly, we extend the platoon size to 12 vehicles and present quantitative results in Fig.~\ref{fig: FF quantitative results}, which plots the $\mathcal{L}_2$ norms of $e_{lat}^{des}$ and $\mathbf{e}^{des}$ in panels (a) and (b), respectively. Both norms exhibit a strict increase from one vehicle to the next, once again confirming the instability results.
\begin{figure}[h]
    \centering
    \subfigure[$\mathcal{L}_2$ norm of $e_{lat}^{des}$]{\includegraphics[width=0.239\textwidth]{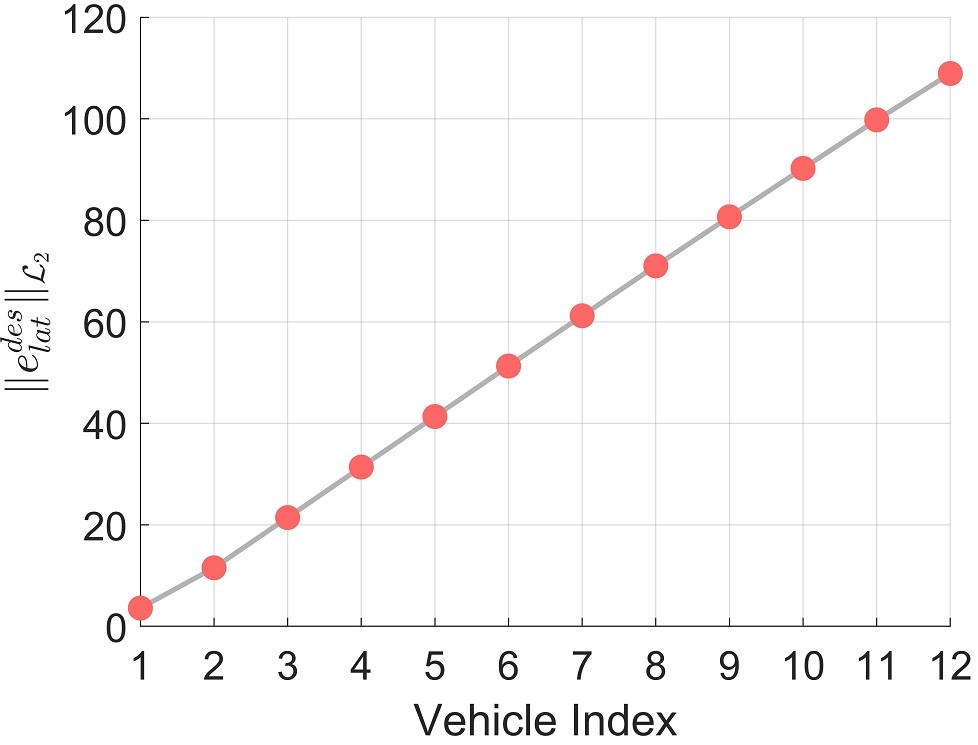}} 
    \subfigure[$\mathcal{L}_2$ norm of $\mathbf{e}^{des}$]{\includegraphics[width=0.239\textwidth]{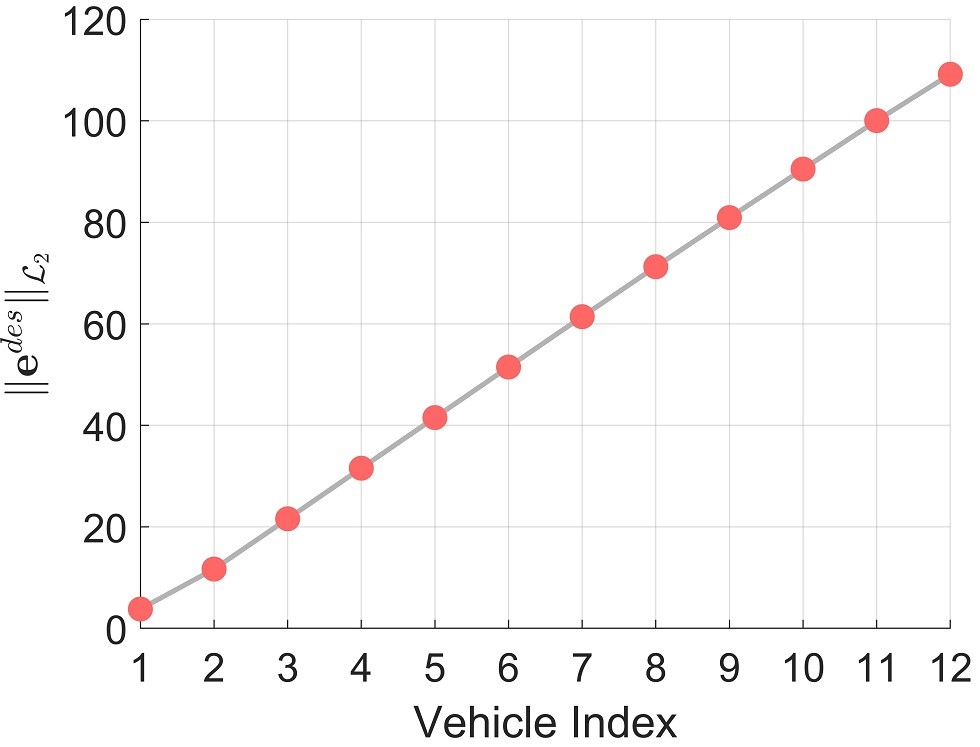}}
    \caption{Quantitative results for FF with onboard sensing}
    \label{fig: FF quantitative results}
    \vspace{-0.5em}
    \end{figure}

\section{Conclusion}
\label{sec7}
This work provided an arc-length parameterized formulation of the platoon lateral control problem and defined  $\mathcal{L}_2$ lateral string stability. Within this framework, we examined two information acquisition modes, each with a corresponding control strategy, and two choices of outputs: specifically, onboard sending vs. V2V, FF vs. LFP control, and tracking error vector vs. lateral error. We analyzed four representative combinations for  $\mathcal{L}_2$ lateral string stability.

Our analysis revealed that only one of the four combinations can achieve  $\mathcal{L}_2$ lateral string stability. In particular, the LFP strategy with V2V with output $\mathbf{y_i}=e_i^{lat}$ is the sole case that ensures error attenuation along the platoon. These results have important implications for lateral controller design in vehicle platoons, where the onboard sensors of the following vehicles can be occluded. Our findings underscore the critical role of incorporating desired path information and predecessor information into the design through V2V. The findings were validated through comprehensive numerical simulations.

\addtolength{\textheight}{-12cm}   




\bibliographystyle{IEEEtran}
\bibliography{IEEEabrv,root}

\end{document}